\documentclass{article}

\newcommand{\fw}{FLAKE}

\usepackage[preprint]{neurips_2023}

\usepackage[utf8]{inputenc} % allow utf-8 input
\usepackage[T1]{fontenc}    % use 8-bit T1 fonts
\usepackage{hyperref}       % hyperlinks
\usepackage{url}            % simple URL typesetting
\usepackage{booktabs}       % professional-quality tables
\usepackage{amsfonts}       % blackboard math symbols
\usepackage{nicefrac}       % compact symbols for 1/2, etc.
\usepackage{microtype}      % microtypography
\usepackage{xcolor}         % colors

\usepackage{graphicx, caption}
\usepackage{subcaption}

\usepackage{amssymb}
\usepackage[export]{adjustbox}
\usepackage{nicematrix}
\usepackage{tikz}
\usetikzlibrary{fit}

\usepackage{amsmath}
\usepackage{amssymb}
\usepackage{mathtools}
\usepackage{amsthm}

\usepackage{paralist}

\usepackage[capitalize,noabbrev]{cleveref}

\theoremstyle{plain}
\newtheorem{theorem}{Theorem}[section]

\theoremstyle{definition}

\theoremstyle{remark}

\usepackage[textsize=tiny]{todonotes}

\title{A Privacy-Preserving Federated Learning Approach for Kernel methods}
\author{%
  Anika Hannemann \thanks{Dept. of Computer Science, Leipzig University} \hspace{0.05mm} \thanks{Center for Scalable Data Analytics and Artificial Intelligence (ScaDS.AI) Dresden/Leipzig} \\
  \texttt{anika.hannemann@informatik.uni-leipzig.de} \\
   \And
   Ali Burak Ünal \thanks{Medical Data Privacy Preserving Machine Learning (MDPPML), University of Tuebingen} \hspace{0.05mm} \thanks{Institute for Bioinformatics and Medical Informatics (IBMI), University of Tuebingen} \\
   \texttt{ali-burak.uenal@uni-tuebingen.de} \\
    \And
    Arjhun Swaminathan \footnotemark[3] \hspace{0.05mm} \footnotemark[4] \\
    \texttt{arjhun.swaminathan@uni-tuebingen.de} \\
    \And
    Erik Buchmann \footnotemark[1] \hspace{0.05mm} \footnotemark[2] \\
    \texttt{buchmann@informatik.uni-leipzig.de} \\
   \And
   Mete Akgün \footnotemark[3] \hspace{0.05mm} \footnotemark[4] \\
   \texttt{mete.akguen@uni-tuebingen.de} \\
}

\begin{document}

\maketitle

\begin{abstract}
% this should be already known by anyone reading this paper ;)
% Kernel methods are a prominent class of machine learning algorithms. However, it is challenging to implement such methods, if the data sources are distributed and cannot be joined at a trusted third party for privacy reasons. 
It is challenging to implement Kernel methods, if the data sources are distributed and cannot be joined at a trusted third party for privacy reasons. It is even more challenging, if the use case rules out privacy-preserving approaches that introduce noise. % to the kernel. 
An example for such a use case is machine learning on clinical data. To realize exact privacy preserving computation of kernel methods, we propose FLAKE, a Federated Learning Approach for KErnel methods on horizontally distributed data. With FLAKE, the data sources mask their data so that a centralized instance can compute a Gram matrix without compromising privacy. 
The Gram matrix allows to calculate many kernel matrices, which can be used to train kernel-based machine learning algorithms such as Support Vector Machines. We prove that FLAKE prevents an adversary from learning the input data or the number of input features under a \textit{semi-honest} threat model. Experiments on clinical and synthetic data confirm that FLAKE is outperforming the accuracy and efficiency of comparable methods. The time needed to mask the data and to compute the Gram matrix is several orders of magnitude less than the time a Support Vector Machine needs to be trained. Thus, FLAKE can be applied to many use cases. 
\end{abstract}

\section{Introduction}
Kernel methods are a prominent class of machine learning algorithms. However, in many real-world scenarios, kernel methods such as Support Vector Machines (SVM) cannot be readily applied, because the data sources are inherently distributed, but the data is private and cannot be shared freely. 
Consider a machine learning scenario, where a Kernel method on medical data is to be used to develop effective treatments, or to identify risk factors for certain diseases.
The input data is collected from multiple hospitals, and it carries sensible medical information that must be kept private. In this scenario it is impossible to apply noise, because neither the patient nor the physician can accept stochastic results. The delay due to processing strong cryptography on a large data set in multiple rounds of a Secure-Multiparty Computation Protocol is also unacceptable. 

Existing work in the field of Secure-Multiparty Computation~\citep{mugunthan2019smpai,zhang2022augmented} or Privacy-Aware Federated Learning~\citep{pfitzner2021federated,adnan2022federated,malekzadeh2021dopamine} can be categorized in three approaches based on \textbf{(1)} encryption \textbf{(2)} differential privacy or \textbf{(3)} randomized masking \citep{monreale2016privacy}. The first two either apply strong cryptography or add noise to private data which is a severe restriction for many use cases.
In this paper, we focus on the third approach using randomized masking. 
We present FLAKE, our Federated Learning Approach for KErnel methods. FLAKE computes the Gram matrix over distributed data sources that store horizontally partitioned data. The Gram matrix allows various kernel matrices to be computed and kernel-based machine learning algorithm to be trained as if the training takes place on centralized data. Examples for such algorithms include Support Vector Machines, Gaussian processes, kernel k-means, and more. To ensure privacy, FLAKE masks the input data at the sources. FLAKE ensures that the resulting Gram matrix is exact. In order to update the Gram matrix, only a fraction of the values need to be re-calculated. Thus, inference and update are inexpensive operations. 
%
% contributions
% In this paper, 
We make three contributions: 
\begin{compactenum}
\item We introduce the FLAKE protocol, which allows a function party to privately compute a Gram matrix on masked input data from multiple input parties. 
\item We prove that both the input data and the number of features is kept private, unless function party and input parties collude and share unmasked data. 
\item We evaluate FLAKE by experiments with medical and synthetic data. % , which is a typical use case for our approach. 
\end{compactenum}

% what are we proud of?
Our formal analysis and our experiments confirm that FLAKE has the potential to open up new fields of application for Kernel-based methods on horizontally partitioned data, that must be kept private, but must be analyzed with an exact approach. 

% paper structure
Paper structure: 
Section~\ref{sec:background} reviews related work, followed by a description of FLAKE in Section~\ref{sec:FLAKE}. Section~\ref{sec:proofs} analyzes the privacy properties of the protocol. Section~\ref{sec:experiments} contains the experimental evaluation. 
Finally, Section~\ref{sec:conclusion} concludes.

\section{Related Work}
\label{sec:background}

\subsection{Kernel-based Methods}
%WHAT IS SUPPOSED TO GO HERE
%- Introduction into Kernel-based Methods \\
%- within all those Kernel Methods there are dot product kernel methods like polynoial, gaussian, ... \\
%- they can be computed with purely the dot products of the data set, they dont need the raw data \\
%- From the dot products, we can then perform methods like support vector machines for regression and classification, support vector clustering (SVC) and PCA. \\
%- In this paper, just classification with an svm is used for experiments 

%%% Alternative text for Kernel methods subsection
Kernel-based machine learning algorithms have a well-established mathematical background. They are among the well-performing machine learning algorithms and are widely utilized in various applications \citep{morota2014kernel,haywood2021kernel}. They can learn non-linear patterns in the data efficiently thanks to the kernel trick: the data is represented by a set of pairwise similarity comparisons, the kernel values, instead of explicitly mapping them into higher dimensions, where linear classification can be done. To compute these kernel values, one can use several different kernel functions such as linear, polynomial, and radial basis function (RBF). Both polynomial and RBF kernels can be computed by using the kernel matrix of linear kernel, which is the Gram matrix. The Gram matrix is a positive semi-definite matrix and its entries indicate the dot product of the corresponding samples' feature vectors. 
Therefore, we can formulate both kernels such that they are computable by using the entries of the Gram matrix. For instance, the polynomial kernel can be written as $k(x,y) = (x^Ty + v)^p$ where $v \leq 0$ is a trade-off parameter and $p \in \mathbb{N}$ is the degree of the polynomial. Similarly, the RBF kernel can be formulated as $k(x,y) = exp(-\dfrac{{\|x^Tx - 2 x^Ty + y^Ty\|}^2}{2 \sigma^2})$ where $\sigma \leq 0$ is the similarity adjustment parameter. In \fw{}, we will benefit from this observation to compute the desired kernel matrices from the Gram matrix.

\subsection{Federated Learning}
 Introduced by  \citep{mcmahan2017communication}, Federated Learning (FL) allows users to reap the benefits of modeling on rich yet sensitive data stored on distributed nodes. In conventional machine learning, a model $\mathcal{M}$ is trained by the centralized data $\mathcal{D}_{cent}$. However, due to privacy concerns, the data is not allowed to leave the nodes. FL addresses this problem. Participating nodes $\mathcal{N}_1, ..., \mathcal{N}_n$ in FL aim to collaboratively train the model $\mathcal{M}$ without revealing their data to other nodes. In FL, every node $\mathcal{N}_i$ trains a local model $\mathcal{M}_i$ on its respective data set $\mathcal{D}_i$ and subsequently shares the model parameters with a central server. The central server then aggregates the received model parameters to obtain a global model $\mathcal{M}_{fed}$ with an accuracy of $acc_{fed}$. As more data is collected, the process is repeated, with each node updating its local model and forwarding the updated parameters to the central server. Thus, the data does not leave its origin at any time in the computation. At some point in the iteration of FL, if $\vert acc_{fed} - acc_{cent} \vert \leq \delta$, then the Federated Learning framework is said to have $\delta$-accuracy loss. The goal in FL is to have less accuracy loss while maintaining efficiency and the data's privacy.

% In addition to the privacy of the individual models of the nodes and their data, there are several studies using different privacy preserving techniques to address the privacy of the aggregated model. 
The privacy of the aggregated models can be ensured in different ways.
Approaches based on \textbf{encryption (1)} like homomorphic encryption (HE) \citep{wibawa2022homomorphic,stripelis2021secure} aim to protect the privacy of aggregated models by encrypting individual models, but HE is computationally heavy and limited in functionality. Another cryptographic approach is secure multi-party computation (SMC) \citep{mugunthan2019smpai,zhang2022augmented}, which allows multiple parties to jointly compute on private data without revealing it, but SMC still requires significant execution time due to communication overhead. 

FL studies utilizing methods based on \textbf{differential privacy (2)} (DP), protect the privacy of the aggregated model by adding noise to the individual models, making it impossible to restore the original model or to inference information about a data point's membership. However, this usually involves a cutback in accuracy \citep{pfitzner2021federated,adnan2022federated,malekzadeh2021dopamine}.

The \textbf{randomized masking approach (3)} for FL was used by \citep{chen2005privacy, chen2007towards} who propose a geometric perturbation approach to preserve data privacy in classification tasks by hiding content while maintaining dot product and Euclidean distance relationships. To provide even stronger security, \citep{lin2015secure} utilize a random linear transformation scheme that requires the data owner to send perturbed data to the service provider for training SVM classifiers. Lin also applies perturbation for clustering tasks using a randomized kernel matrix to hide dot product and distance information \citep{lin2013privacy}. Another randomization technique using Bloom filters enables outsourcing of mining association rules while protecting business intelligence and customer privacy, but only supports approximate reconstruction of mined frequent item sets by the data owner \citep{qiu2008protecting}. \citep{yu2006privacy} introduce random kernels where the original data gets transformed using random linear transformation. However, due to the nature of approximation and introduction of noise, they all suffer from performance loss to provide privacy. 
\citep{unal2021escaped} provide an exact protocol and is, therefore, the closest study to our approach. Here, the data sources first have to communicate with each other to mask their data. Then they send these masked samples to enable the cloud so that it can compute the desired kernel-based machine learning algorithm. However, due to the utilized encoding technique in ESCAPED, one has to run the protocol from scratch whenever there is new data in any party that needs to be integrated into the model or there is a new party involved in the computation. % In \fw{}, updating a Gram matrix is an inexpensive and fast operation. 

\section{FLAKE}
\label{sec:FLAKE}
% In this section, we will first outline the scenario employed in the paper. Then, we introduce the proposed framework FLAKE. Finally, we discuss the respective security definition as well as perform a security analysis.
% In this section, we introduce our scenario, explain our privacy-aware Federated Learning Approach for KErnel methods (FLAKE), and we analyze its privacy.

This section explains FLAKE, our privacy-aware Federated Learning Approach for KErnel methods. 

\subsection{Scenario Definition}
%scenario: 
%- multiple input parties, one function party
%- real life use cases
%- Federated Learning approach 

% tried to write the scenario a little shorter, and leave hints towards updateability / efficiency
We assume a multi-party scenario consisting of multiple input parties (Alice, Bob, Charlie for simplicity) and one function party. Alice, Bob and Charlie hold sensitive data that is horizontally partitioned, i.e., each input party stores the same schema with different training data. The function party performs Federated Learning iteratively on a (possibly large) set of input-data chunks.\\ We consider a fully untrusted setting where the input data must not leave their origin. Formally, we assume an arbitrary subset of semi-honest input parties and function party where no party colludes with another one. Note that this leaves aside extreme data distributions or all-zero cases where properties of the training data of one or more input parties can be guessed, or where only one input party exist.
Therefore, FLAKE needs to deal with four requirements: 
\begin{tabbing}
    Updatability: \=  \kill
\emph{Privacy:} \> The function party or an input party cannot learn the data of another input party, and \\ \>the number of features is kept private from the function party.\\
\emph{Accuracy:} \> The accuracy of the federated model must be as good as that of the centralized one.\\
\emph{Updatability:} \> It must be possible to update the model with new data.\\
\emph{Efficiency:} \> Communication costs and execution time must be feasible for our scenario.\\
\end{tabbing}

\subsection{The FLAKE Protocol}

%\begin{figure*}[t]
%\begin{minipage}{.5\textwidth}
%	\centering
%  \includegraphics[width=0.95\columnwidth, frame]{pictures/method1.png}
% \caption{Computing matrix $N$ and left inverse $L$ of $N$.}
%\label{fig:method1}
%\end{minipage}
%\begin{minipage}{.5\textwidth}
%	\centering
%  \includegraphics[width=0.95\columnwidth, frame]{pictures/method2.png}
 % \caption{Computing the dot products from masked $A, B, C$.}  \label{fig:method2}
%\end{minipage}
%\end{figure*}

\begin{figure*}
    \centering
    \includegraphics[width=0.8\linewidth]{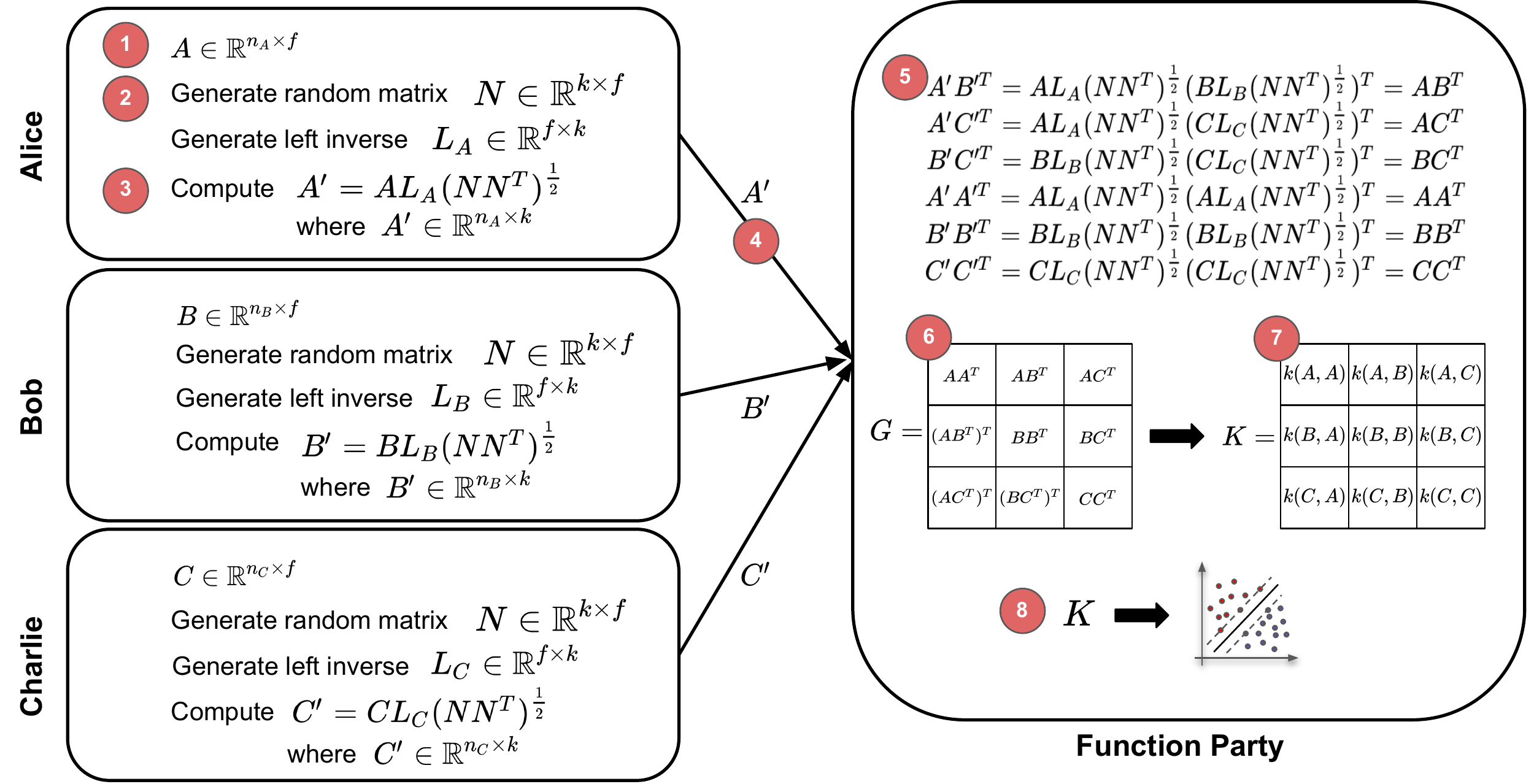}
    \vskip 0.1in
    \caption{Masking and Training with FLAKE: \textbf{(1)} Initially, each party has its own data. \textbf{(2)} A random matrix $N$ and its left inverse $L$ are computed, based on a shared seed. \textbf{(3)} According to protocol, the data gets masked. \textbf{(4)} Masked data is transferred. \textbf{(5)} The function party computes the dot products. \textbf{(6)} The Gram matrix is formed by the dot products and their transposed values.  \textbf{(7)} A kernel matrix is computed from the Gram matrix. \textbf{(8)} Finally, a classifier is trained.}  
    \label{fig:flake_overview}
\end{figure*}

% this paragraph sounds rather reptetitive, because it is very similar to what is said in the following paragraphs. To save space, i shortened it generously
%Intuitively, FLAKE is a \emph{privacy-aware}, \emph{accurate}, \emph{iterative} and \emph{efficient} approach to compute the Gram matrix of samples from different input parties to enable the training and testing of kernel-based machine learning algorithms. This takes place in three stages \textit{Distribution of Seed}, \textit{Masking and Training} and \textit{Inference and Updating}. 
% %
% In the first stage (Distribution of Seed) a common seed is exchanged. %This is itself not part of the protocol, but a necessary initialization step.
% In the second stage (Masking and Training), the input data is masked with pseudo-inverses.
% In particular, the input parties compute a random matrix $N$ based on a shared seed, and use singular value decomposition to obtain the left inverse $L$ of $N$. 
% Then the masked data is sent to the function party, which computes the dot products (cf.~Figure~\ref{fig:flake_overview}).
% %
% For the third stage (Inference and Updating), FLAKE provides a protocol for inferencing and updating the Gram matrix. 

FLAKE computes the Gram matrix of samples from different input parties to enable the training and testing of kernel-based machine learning algorithms. This takes place in three stages \textit{Distribution of Seed}, \textit{Masking and Training} and \textit{Inference and Updating}.

%\begin{figure} [t]
%\centering
%\begin{minipage}{.5\textwidth}
%  \centering
%\begin{minipage}{.95\textwidth}
%  \centering
%  \includegraphics[width=0.95\columnwidth, frame]{pictures/method1.png}
% \caption{First step of computing the Gram matrix with FLAKE in a three-party setting: The input parties compute a random matrix $N$ based on the shared seed. With SVD, the left inverse $L$ of $N$ gets computed.}
%\label{fig:method1}
%\end{minipage}%
%\end{minipage}%
%\begin{minipage}{.5\textwidth}
%  \centering
%\begin{minipage}{.95\textwidth}
%  \centering
%%  \includegraphics[width=0.95\columnwidth, frame]{pictures/method2.png}
 % \caption{Second step of computing the Gram matrix with FLAKE in a three-party setting: $A, B, C$ get masked and send to the function party. With $A', B', C'$, the function party is then able to compute the dot products.}  \label{fig:method2}
%\end{minipage}
%\end{minipage}
%\end{figure}

\paragraph{Distribution of Seed} \label{stage0}
FLAKE relies on a  Public Key Infrastructure, which delivers each input party the public signing keys for all other input parties. To initiate the process, one input party is randomly selected as the leader and generates a random seed. The leader then shares this seed with the other input parties using public-key encryption and digital signatures. The function party is a natural choice for the task of the aggregator, which transmits encrypted messages between input parties, but cannot decrypt or modify these messages. We assume a trusted third party for the distribution of public keys. This is a common assumption in frameworks for privacy-preserving federated learning \citep{bonawitz2017practical,zhang2020batchcrypt}. 

% Our system relies on a Public Key Infrastructure, in which each input party receives public signing keys for all other input parties from a trusted third party. To initiate the process, one input party is randomly selected as the leader and generates a random seed. The leader then shares this seed with the other input parties using public-key encryption and digital signatures. Transmission of encrypted messages between input parties is facilitated by the aggregator, which cannot decrypt or modify these messages. The use of a trusted party for the distribution of public keys is a common assumption in frameworks for privacy-preserving federated learning \citep{bonawitz2017practical,zhang2020batchcrypt}. 

\paragraph{Masking and Training} \label{stage1}
%We differentiate in training and inference \& updating. In order to fulfill R1, FLAKE follows the protocol for training as described below. 
% To satisfy R3, the inclusion of additional data is done, as described in the according subsection. Analogously, inference can be performed, i.e. the processing, such as classifying, of data points not used for training. 
% \subsubsection{Training}
% The goal of FLAKE is to compute the Gram matrix for further computation of kernel-based methods. 

The objective of this stage is to let the function party compute a Gram matrix without learning the data from the input parties (Requirement \emph{Privacy}). 
The Gram matrix $G$ of the data matrices $A, B, C$ provided by Alice, Bob, and Charlie is the matrix of all possible inner products $AB^T, AC^T, CA^T,...$. 
For better understanding, we introduce the private calculation of $AB^T$ given $A \in \mathbb{R}^{n_A\times f}$ and $B \in \mathbb{R}^{n_B \times f}$ where $f > 1$. The following protocol reveals $AB^T$ to the function party while hiding the input data and the number of features:

First, Alice and Bob calculate a random full-rank matrix $N \in \mathbb{R}^{k \times f}$ for some $k> f$, based on a shared seed. Throughout all input parties and training iterations, $N$ remains constant. Since $rank(N)=f$, there exists a non-unique matrix $L\in\mathbb{R}^{f \times k}$ such that $LN = I_{f\times f}$, that can be computed using the singular value decomposition (SVD) of N, and is called the Moore-Penrose pseudoinverse. SVD allows us to write $N$ as $N=USV^T$ with $U, V$ being orthogonal matrices and $S$ being a diagonal matrix. The inverse of $N$ can be determined from the SVD: $L=N^{-1}=US^{-1}V^T$. Here, $S^{-1}$ can be derived by taking the multiplicative inverse of every entry of $S$. 
Now, Alice computes a independent left inverse $L_A$ such that $L_AN=I$ and Bob $L_B$ such that $L_BN=I$. Then, the data gets masked by Alice $A'=AL_A(NN^T)^{\frac{1}{2}} \in \mathbb{R}^{n_A \times k}$ while Bob masks his data accordingly $B'=BL_B(NN^T)^{\frac{1}{2}}\in \mathbb{R}^{n_B \times k}$. 
Figure~\ref{fig:flake_overview} illustrates this. 

% no forward reference if not absolutely necessary
%As later shown in the privacy analysis, $A, B$ and the number of respective features are masked by the computations made and can be privately forwarded. $A', B'$ reveal nothing but $n_a, n_b$ and the Gram matrix of the samples. The function party can now compute the dot product needed for further classification $A'B'^T=AB^T$. The remaining entries of the Gram matrix are masked analogously. When dealing with more than three parties, the Gram matrix has to be extended correspondingly. Figure \ref{fig:method1} and Figure \ref{fig:method2} show an overview of FLAKES training process.

$A'$ and $B'$ are forwarded to the function party, which only reveals $n_a$ and $n_b$, respectively, and the Gram matrix of $A$ and $B$ when $A'B'^T$ is computed. The function party computes the dot product $AB^T$, as shown in Figure~\ref{fig:flake_overview}. Then, the function party can compute the desired kernel matrix using the Gram matrix and perform training and testing of the designated kernel-based machine learning algorithm. The remaining entries of the Gram matrix are masked analogously. When dealing with more than three parties, the Gram matrix has to be extended correspondingly.

\paragraph{Inference and Updating}
\label{stage2}

To integrate new data without having to rebuild the model from scratch (Requirement \emph{Updatability}), FLAKE provides a protocol for inference and updating the Gram matrix. 
We can distinguish two cases: First, one of the input parties may have received new input data. Second, a new input party shall be integrated into the computation. 
For simplicity, we again explain our protocol with three parties Alice, Bob and Charlie with their respective data sets $A, B, C$.

Assume $C$ has new data $X$ which must be integrated into the Gram matrix shown in Table~\ref{t:tabularnote}. $X$ is the data set to be used for updating the model. To extend the gram matrix with the new values from $C$, the function party only needs to have the entries in the dashed rectangles. The party $C$ uses the aforementioned masking and sending approaches for this purpose. 
Now assume that a new input party needs to be added. In this case, the function party must calculate the values in the continuous rectangles in Table~\ref{t:tabularnote}. The remaining new entries can be computed locally by $C$. In both cases, updating the Gram matrix means that the function party has to calculate only a small set of new values. The vast majority of values need to be calculated just once, and a large share of the calculation effort remains at the input parties.  
Note that $X$ can be also a test data set. 

When a party wants to leave the consortium the function party deletes all random components coming from this party and gram matrix entries that are calculated using these random components. This is important for compliance with legal regulations such as General Data Protection Regulation (GDPR) \citep{gdpr}. It can be seen as an application of machine unlearning. In current FL methods, it is unclear and difficult how to eliminate a party's contribution to the collaboratively trained ML model.

\begin{table}[h]
\setlength{\belowcaptionskip}{1ex}
\centering
\small
\caption{Gram matrix of three-input parties.}
% \vskip 0.1in
\label{t:tabularnote}
{%
\NiceMatrixOptions{cell-space-limits = 1mm}
\begin{NiceTabular}{*5{c}}[name=MyTbl5]

    \hline
    Input Parties & A & B & C & X\\\hline
    A   & $AA^T$   & $AB^T$ & $AC^T$ & $AX^T$  \\
    B   & $BA^T$   & $BB^T$  & $BC^T$  & $BX^T$ \\ 
    C   & $CA^T$   & $CB^T$  & $CC^T$  & $CX^T$  \\
    X   & $XA^T$   & $X B^T$  & $XC^T$   & $XX^T$ \\ \hline
    
    \CodeAfter
        \tikz \node[draw, color=red, rounded corners, inner ysep=0.25mm, rectangle, fit=(MyTbl5-5-2) (MyTbl5-5-3)] {};
        \tikz \node[draw, color=red, rounded corners, inner ysep=0.25mm, inner xsep=2mm, rectangle, fit=(MyTbl5-2-5) (MyTbl5-3-5)] {};   

        \tikz \node[draw, color=blue, dashed, rounded corners, inner ysep=-0.5mm, rectangle, fit=(MyTbl5-5-2) (MyTbl5-5-4)] {};
        \tikz \node[draw, dashed, color=blue, rounded corners, inner ysep=-0.5mm, inner xsep=0.5mm, rectangle, fit=(MyTbl5-2-5) (MyTbl5-4-5)] {};   
\end{NiceTabular}%
}
\end{table}

\section{Analysis of Privacy Properties}
\label{sec:proofs}

%In this section, we prove that FLAKE fulfills Requirement \emph{Privacy}. 
% To this end, we start with a privacy definition.

\subsection{Privacy Definition}
We consider the semi-honest (or honest-but-curious) adversary model. 
In a multi-party scenario, a \textbf{semi-honest adversary}~\citep{evans2018pragmatic} corrupts an arbitrary subset of the parties involved. The corrupted parties follow the multi-party protocol as specified, i.e., the output of the protocol is correct. The corrupted parties try to learn private data from the messages they receive from uncorrupted parties. At the end of the protocol, the corrupted parties are allowed to share their information. 
%A \textbf{malicious adversary}~\citep{evans2018pragmatic} is defined analogously, but it is allowed to deviate from the protocol. Therefore, the output of the protocol may not be correct in the case of a malicious adversary. 

FLAKE consists of a function party and a number of input parties. From Requirement \emph{Privacy} follows that FLAKE needs to ensure two privacy properties: \textbf{(i)} the data of uncorrupted input parties must kept private from any corrupted input party or the function party, and \textbf{(ii)} a corrupted function party must not be able to learn the number of features. 

If the function party and all input parties operate honestly, privacy properties (i) and (ii) are ensured. If all input parties have been corrupted by a semi-honest adversary, privacy cannot ensured. 
Between these extreme cases, we distinguish three cases for further analyses:

\begin{compactitem}
\item[(1)] A subset of the input parties is corrupted by a semi-honest adversary.
\item[(2)] The function party is corrupted by a semi-honest adversary. 
\item[(3)] The function and a subset of input parties are corrupted by a semi-honest adversary.
%\item[(4)] How do the privacy properties of the cases (1)-(3) change when the adversary is malicious?
\end{compactitem}

Recall that we do not consider extreme scenarios. In particular, we exclude data distributions where the number of features or the training data of one or more input parties can be guessed, and protocols with only one input party. However, to make the guessing harder, the input parties generate a unique matrix $L$ in each iteration. Therefore, the function party can not determine if an input party updates their data in a subsequent iteration. Also, all-zero rows are not allowed; though these are usually discarded as part of preprocessing anyway. 

%towards a semi-honest adversary, we have to distinguish only two cases:
%\begin{compactitem}
%\item[(1)] The function party has been corrupted by the adversary. 
%\item[(2)] A subset of input parties has been corrupted by the adversary.
%\end{compactitem}
%Under a semi-honest adversary model, the function party has only access to the Gram matrix and hence poses no threat to the security of the data. Case (2) will be considered in Theorem \ref{theorem1}. 

%A \textbf{malicious adversary}~\citep{evans2018pragmatic} also corrupts an arbitrary subset of the parties. 
%However, the corrupted parties can deviate from the protocol in any way they choose. This also includes a collaboration between corrupted parties. Again, the adversary attempts to learn the input data and/or the number of features of the uncorrupted input parties.

\subsection{Privacy Analysis}

Before we begin analysing the privacy of the protocol, we shall establish its correctness, which is unaffected by the existence of a semi-honest adversary.
\begin{proof}
Without loss of generality, we assume there are two input parties Alice and Bob with individual left inverses $L_A$ and $L_B$ of a common mask matrix $N$, whose outputs are $A'=AL_A(NN^T)^{\frac{1}{2}}$ and $B'=BL_B(NN^T)^{\frac{1}{2}}$. Then, the correctness of the protocol follows as below.
    \begin{equation*}
        \begin{split}A'B'^T&=AL_A(NN^T)^{\frac{1}{2}}(BL_B(NN^T)^{\frac{1}{2}})^T, \\ 
        &=AL_A(NN^T)^{\frac{1}{2}}(NN^T)^{\frac{1}{2}}L_B^TB^T,\\
        &=AL_A(NN^T)L_B^TB^T,\\
        &=A(L_AN)(L_BN)^TB^T,\\
        &=AB^T = (BA^T)^T.
        \end{split}
    \end{equation*}
Analogously, correctness follows for $AA^T$ and $BB^T$.
% Analogously, the correctness follows for $BC^T$ and $CA^T$. Since the input parties do not communicate with each other, an adversary only has access to the data of the corrupted parties, and hence this poses no threat to the security of the data of the non-corrupted parties. 
\end{proof}

We analyze \textbf{Case (1)} first. Since the input parties know the number of features, we only have to prove Property~(ii), i.e., 
a corrupted function party cannot learn the number of features.

\begin{theorem} \label{theorem2}
FLAKE is secure against a semi-honest adversary who corrupts a subset of the input parties. 

\end{theorem}
\begin{proof}
%In the FLAKE protocol, the input parties refrain from engaging in direct communication with each other, except for the establishment of a shared seed, which facilitates the generation of the masks. Intrinsically, in this setting, a semi-honest adversary possesses the capability to access the information belonging to the compromised parties only. Consequently, the data acquired from these compromised entities does not provide the adversary with sufficient means to deduce anything about the data from the honest parties since it can not access their masked data. Further, as shown above, in the semi-honest adversary setting, the protocol proceeds as intended.
% Since the adversary is semi-honest, it can only read the data of the corrupted input parties while following the protocol as intended. It is then easy to conclude the correctness of the method. 
% AN ALTERNATIVE FOR THE PROOF OF THIS THEOREM
Let $S_U$ be the set of all input parties involved in the computation. While executing FLAKE protocol, an input party $P \in S_U$ has access only to the common mask N, the common seed used to generate N and the left inverse $L_P$ of N generated by P. At any point in FLAKE protocol, the input party P gets neither the masked data of other input parties nor the computed Gram matrix using the masked data of all input parties. Thus, A semi-honest adversary corrupting a subset of input parties $S_C \subset S_U$ cannot learn the data of non-corrupted input parties $S_H \subset S_U$ where $S_C \cap S_H = \emptyset$.

FLAKE is, therefore, secure against the semi-honest adversary corrupting a subset of input parties. Because a semi-honest adversary follows the protocol, the data provided by the corrupted input parties do not affect the result of the computation. \end{proof}

Regarding \textbf{Case (2)}, we need to prove that FLAKE does not allow a semi-honest function party to learn (i) input data nor (ii) the number of features. 

\begin{theorem} \label{theorem3}
    FLAKE is secure against a semi-honest adversary who corrupts the function party.
\end{theorem}
\begin{proof}
A semi-honest function party is only the receiver of the masked data from the input parties, and follows the protocol as intended. Without loss of generality, let there be two input parties Alice and Bob with input data $A \in \mathbb{R}^{n_A \times f}$ and $B \in \mathbb{R}^{n_B \times f}$, respectively, where $n_x$ is the number of samples in the corresponding party and $f$ is the number of features. The semi-honest function party receives the masked input matrices of them, which are $A'=AL_A(NN^T)^{\frac{1}{2}} \in \mathbb{R}^{n_A \times k}$ and $B'=BL_B(NN^T)^{\frac{1}{2}} \in \mathbb{R}^{n_B \times k}$ where $k > f$. Then, it computes $A'B'^T = AB^T \in \mathbb{R}^{n_A \times n_B}$, $A'A'^T = AA^T \in \mathbb{R}^{n_A \times n_A}$ and $B'B'^T = BB^T \in \mathbb{R}^{n_B \times n_B}$. The data that the function party has access to then includes 
\begin{enumerate}
    \item[$(a)$] $A'$ and analogously, $B'$.
    \item[$(b)$] $AB^T=(BA^T)^T$, $AA^T$ and analogously $BB^T$.
\end{enumerate}
Regarding $(a)$, it is trivial that $A'$ does not reveal the number of features of $A$. We now show that $A'$ is not produced by a unique matrix $A$. Given an orthogonal matrix $O \in \mathbb{R}^{f\times f}$ with $f>1$, for $\tilde{A}=AO$ and $L_{\tilde{A}}=O^TL_A$, we have $A'=\tilde{A}(L_{\tilde{A}}(NN^T)^{\frac{1}{2}}$. Further, since we require that not all entries of any one sample is full of zeroes, the function party cannot deduce anything about $A$ from $A'$. 
\\~\\ 
Regarding $(b)$, the matrices that produce these Gram matrices are not unique, since for any orthogonal matrix $O \in \mathbb{R}^{f\times f}$ where $f>1$, labeling $\tilde{A}=AO$ and $\tilde{B}=BO$, we have
\begin{gather*}
    \tilde{A}\tilde{A}^T=AA^T,\quad
        \tilde{B}\tilde{B}^T=BB^T,\quad
                \tilde{A}\tilde{B}^T=AB^T.
\end{gather*}
In consequence, the function party only learns the singular values and singular vectors of the matrices, i.e., it can find $U$ and $S$ from the singular value decomposition $A=USV^T$ by eigen-decomposing $AA^T$. However, these values are insufficient to solve for $A$ since we can generate countless number of different orthogonal matrices \citep{aguilera2004general}. The function party learns neither (i) input data nor (ii) the number of features. 

% additional argument: why can't an adversary learn anything from the Gram matrix -> reference to ESCAPED paper:  
Although the function party obtains the Gram matrix, it cannot deduce the samples used to compute this Gram matrix, which was shown by \citep{unal2021escaped}. Details can be found in the supplementary material. 

\end{proof}

\textbf{Case (3)} means that not only the function party, but also a subset of the input parties has been corrupted by a semi-honest adversary. In this case, since the adversary knows $N$, the privacy of the data of the other parties is compromised since for data from a non-corrupt party Charlie of the form $C'=CL_C(NN^T)^{\frac{1}{2}}$, the adversary can obtain $C$ by multiplying the data with $(NN^T)^{\frac{1}{2}}L^T$.

\section{Experiments}
\label{sec:experiments}

\subsection{Implementation}
In this section, we evaluate the performance of FLAKE and provide a run-time analysis.

We experiment with three clinical data sets which contain medical records and, thus, have strong privacy concerns \citep{wolberg1992breast, unal2019framework, UCI-ML-Repository}. All of them are suitable for classification tasks. For the run-time analysis, we experimented with a synthetic data set with $\{500, 1000, 2000, 4000, 8000\}$ data points  (dp) for each input party.
Details about their statistics can be found in the supplementary material. 

Before starting with the run-time experiments, we want to compare FLAKE to other methods for randomization-based kernel computation for horizontally shared data. For this purpose, we implemented a 5-fold cross validation with FLAKE, ESCAPED \citep{unal2021escaped}, PPSVM \citep{yu2006privacy}, RSVM \citep{lin2015secure} and a naive SVM classifier in Python. Our experiments show that FLAKE, ESCAPED and the naive classifier produces the same results as they are exact solutions. Because of the introduced stochasticity, RSVM and PPSVM have a performance almost as good as the naive classifier, but they are not exact. Furthermore, the overhead associated with the various methods was measured for a single node and 1000 data points. The overhead for all methods was found to be extremely low, to the point of being negligible. Therefore, the subsequent experiments will primarily focus on scaling up the number of data points and input parties for FLAKE and ESCAPED, the two exact methods. For further details see the supplementary material. 

We implemented FLAKE for a scenario with three input parties and one function party. 
To mimick the network communication between input parties and function party, we have implemented each party as an isolated process that communicates with others via TCP connections. Our four data sets are divided into three disjoint partitions. Each partition is assigned an input party. Each input party then masks its data according to the FLAKE protocol, and splits the masked data into chunks.
After that, each input party compresses the chunks by zlib's Deflate-algorithm, and forwards the compressed chunks to the function party. The function party deflates the chunks, computes the Gram matrix and a polynomial kernel. 
Finally, a SVM is trained with a 5-fold cross-validation. A grid search optimizes the corresponding hyperparameters $C \in \{2^{-4}, ..., 2^{10}\}$ (misclassification penalty) and $p \in \{1, ..., 5\}$ (degree). 

All experiments were executed on a host with an AMD 7713 with 2.0GHZ and 512 GB of memory, which is a typical stand-alone server configuration for a small datacenter. We have used a single-threaded implementation. We repeated each experiment 10 times. 

\subsection{Run-time Analysis}

We want to confirm that training time, masking time, communication time, gram-computation time and update time do not limit the applicability of FLAKE. 
% In this section, we want to find out if the various run times of FLAKE have an impact on its applicability. 
%
As known from literature, SVMs typically do not scale readily to very large data sets. In a centralized scenario, it is the training time for the SVM that limits the size of the input data. 
We declare success, if we can show that the run-times of the stages of FLAKE in a federated scenario are negligible, compared to the stages required for the federated training of a SVM without masking. 
%
% We consider five run-times: 
% \begin{compactitem}
% \item Masking time at an input party.
% \item Time to send masked data from an input party to the function party.
% \item Time to compute the Gram matrix at the function party. 
% \item Training time of the SVM at the function party. 
% \item Update time. 
% \end{compactitem}
%

\paragraph{Training Time}
% We measure the training times first. Note that 
The training takes place at the function party. 
%Our SVM implementation runs in a single thread. Thus, the times for training can be related to the times for computing the Gram matrix or the masking. We do not need to make assumptions about parallelization or the number of CPU cores available at the different parties. 
Figure \ref{runtimes4} shows the training time for varying numbers of dp in our synthetic data set. As expected, the longest takes the training of the data set with 8000 dp with 516.62 ($\pm$ 2.45) on average. Recall that 8000 dp means that each of our three input parties sends a masked data set of this size to the function party.

\paragraph{Masking Time}
To find out how much masking burdens the input parties, we ran a series of experiments, again with the synthetic data set. We varied the number of dp and measured the time for masking. Figure \ref{runtimes1} reports the masking time measured for one input party. Even with 8000 dp per input party, the execution takes less than $0,003$ ($\pm$ 0.0001) seconds on average. This masking time is negligible, compared to the time to train the SVM model, and does not restrict the applicability of FLAKE.

\begin{figure*}
\centering
\begin{subfigure}{.31\linewidth}
  \centering
  \includegraphics[width=\linewidth]{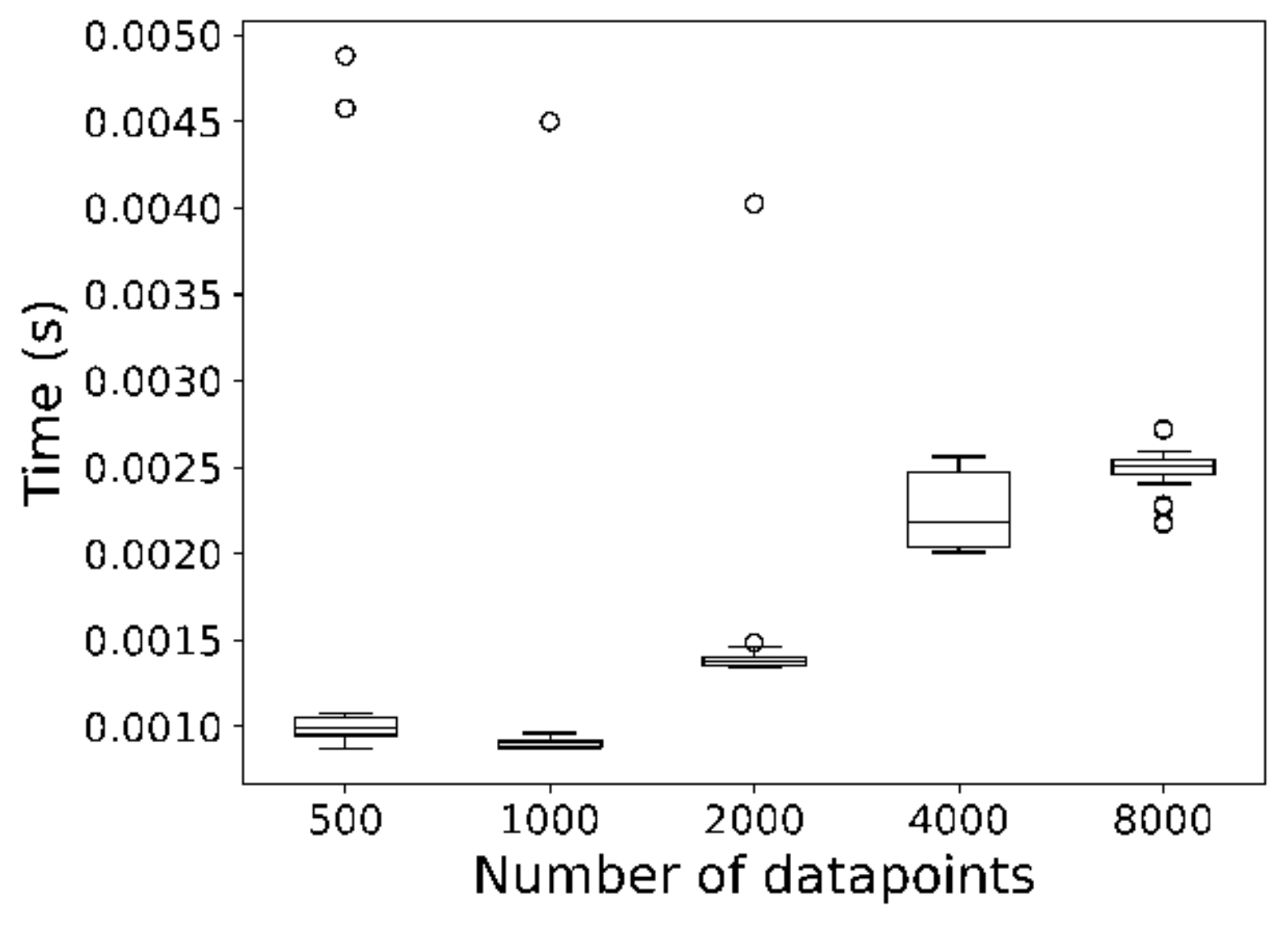}
  \caption{Masking Time}
  \label{runtimes1}
\end{subfigure} \quad
\begin{subfigure}{.31\linewidth}
  \centering
  \includegraphics[width=\linewidth]{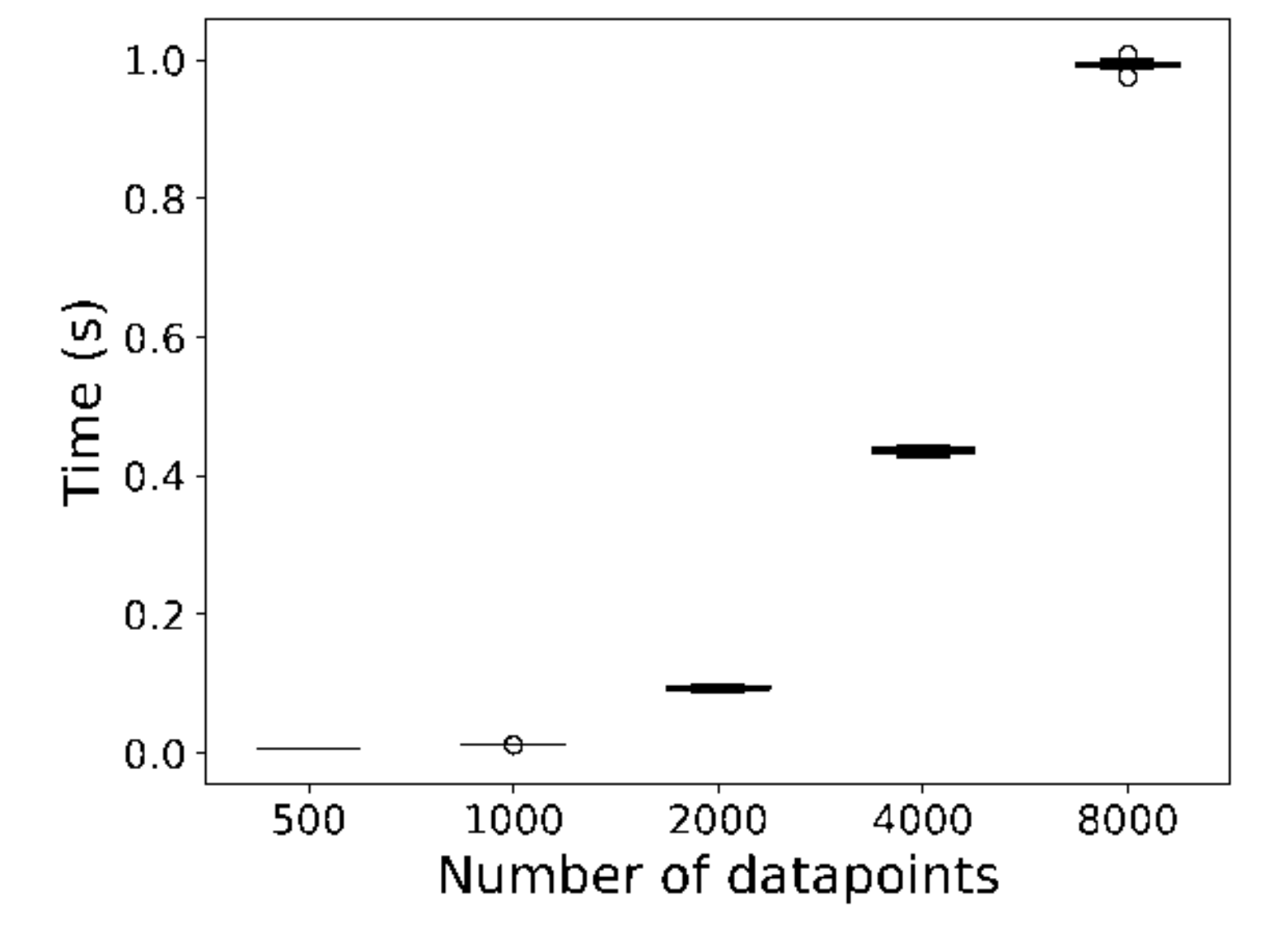}
  \caption{Gram-Computation Time}
  \label{runtimes3}
\end{subfigure} \quad
\begin{subfigure}{.31\linewidth}
  \centering
  \includegraphics[width=\linewidth]{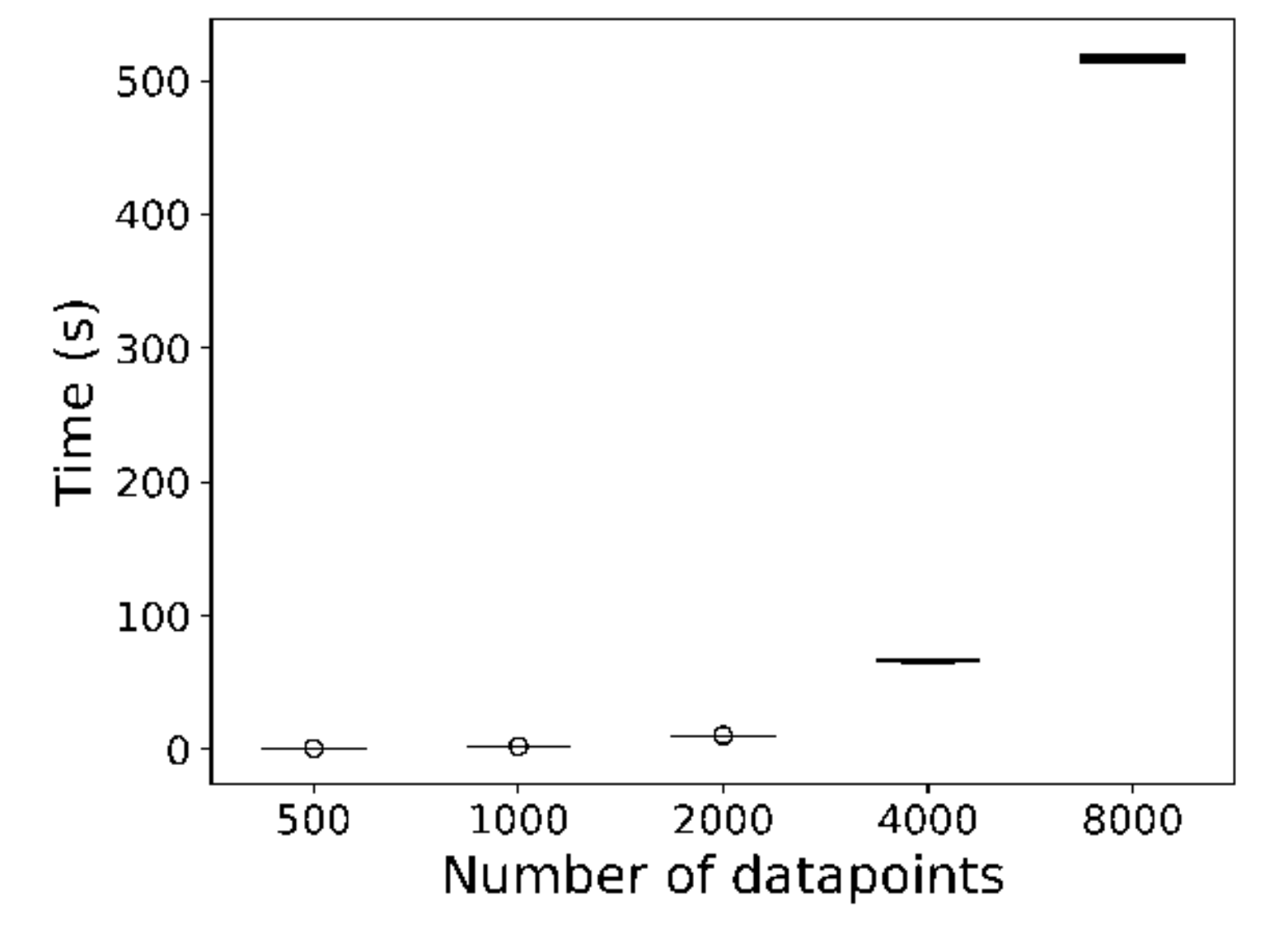}
  \caption{Training Time}
  \label{runtimes4}
\end{subfigure}
\vskip 0.1in
\caption{Run times of FLAKE masking data sets of 1000 - 8000 dp for three input parties each. a) Time to mask the data for one input party. b) Time for computing the Gram matrix from the masked data. c) Training time of SVM}
\label{fig:fig}
\end{figure*}

\paragraph{Communication Time} 
%The time needed to send masked data from the input parties to the function party depends on two aspects: (I) the \textit{communication time} for sending data via network, and (II) the \textit{total transmission time} of the data transfer as a whole. The data transfer starts when an input party splits its masked data into chunks and compresses them for transmission. It ends when the function party receives and decompresses the last chunk, and is ready to compute the Gram matrix.  

Because our implementation runs on a data-center host, we estimate the communication time needed to send masked data from the input parties to the function party. The communication time~$T$ can be estimated as shown in Equation~\ref{eq:transmi}: 
\begin{equation} 
\label{eq:transmi}
    T = \frac{ \text{Datasize}}{\text{Bandwidth}} + \text{Latency} \times (1+ \text{Packetloss})
\end{equation}
Our largest data set consists of 8000 points, which adds up to a Datasize of 1.31 MB for each input party. A typical VPN has a Bandwidth of 1.25MBps, with an average Latency of 0.1s and a Packetloss of 2\% \citep{Ookla2022}. For this set of parameters, the estimated the communication time is 1.05 seconds. Without Latency and Packetloss, it is 1.048 seconds. Recall that our experiments are executed on a single data-center host, i.e., the actual data transfer takes place as inter-process communication in the main memory of the host and virtually requires no time.

\paragraph{Gram-Computation Time}

We also measured the time the function party needs to compute the Gram matrix from the masked data from the input parties. Figure~\ref{runtimes3} shows that the computation time increases slightly more than linearly with the size of the data set, with no outliers. For 8000 dp, it took 0.99 ($\pm$ 0.0083) seconds on average to compute the Gram matrix. Again, 8000 dp means the function party receives 3x8000 masked data sets from our three input parties. In summary, we have confirmed that the Gram-computation time does not contribute much to the total computation time. 

% We already said that
% To demonstrate the efficiency of FLAKE (Requirement 4), the time for masking the data (cf. Figure \ref{runtimes1}) and the time for computing the Gram matrix (cf. Figure \ref{runtimes3}) are crucial. When training a model on 8000 data points for three input parties, the application of FLAKE adds $1.0025$ seconds (avg.) to the training and transmission time of $516+3 \cdot 1.9$ seconds (avg.). It is clear to see, that the generated overhead is insignificant and, thus, FLAKE can be called efficient. 

\paragraph{Update Time}

Having shown that the time required to mask the data, send them to the function party, and compute the Gram matrix is several orders of magnitude below the time to train the model, we now consider updating the model.
To mimick a typical Federated Learning use case, where the training data increases due to dynamic data collection after the initial training, the data sets were updated with additional data in multiple training iterations.

\begin{figure}[ht]
\begin{minipage}[b]{0.5\linewidth}
 In particular, we performed multiple training iterations starting with a synthetic data set with 1000 dp for each input party. Figure \ref{federated} reports the run-times for masking the data and computing the Gram matrix for a three party scenario. We compared four training iterations of FLAKE and ESCAPED \citep{unal2021escaped}, where 1000 dp are added in each iteration. The experiment is measured in the same way as for the other diagrams. The figure confirms that FLAKE outperforms ESCAPED. In particular, masking with ESCAPED takes much more time. We conclude that updating the training data in FLAKE is an inexpensive operation and, thus, can be successfully applied in a FL setting.
\end{minipage}%
\hfill
\begin{minipage}[b]{0.45\linewidth}
\centering
\includegraphics[width=\columnwidth]{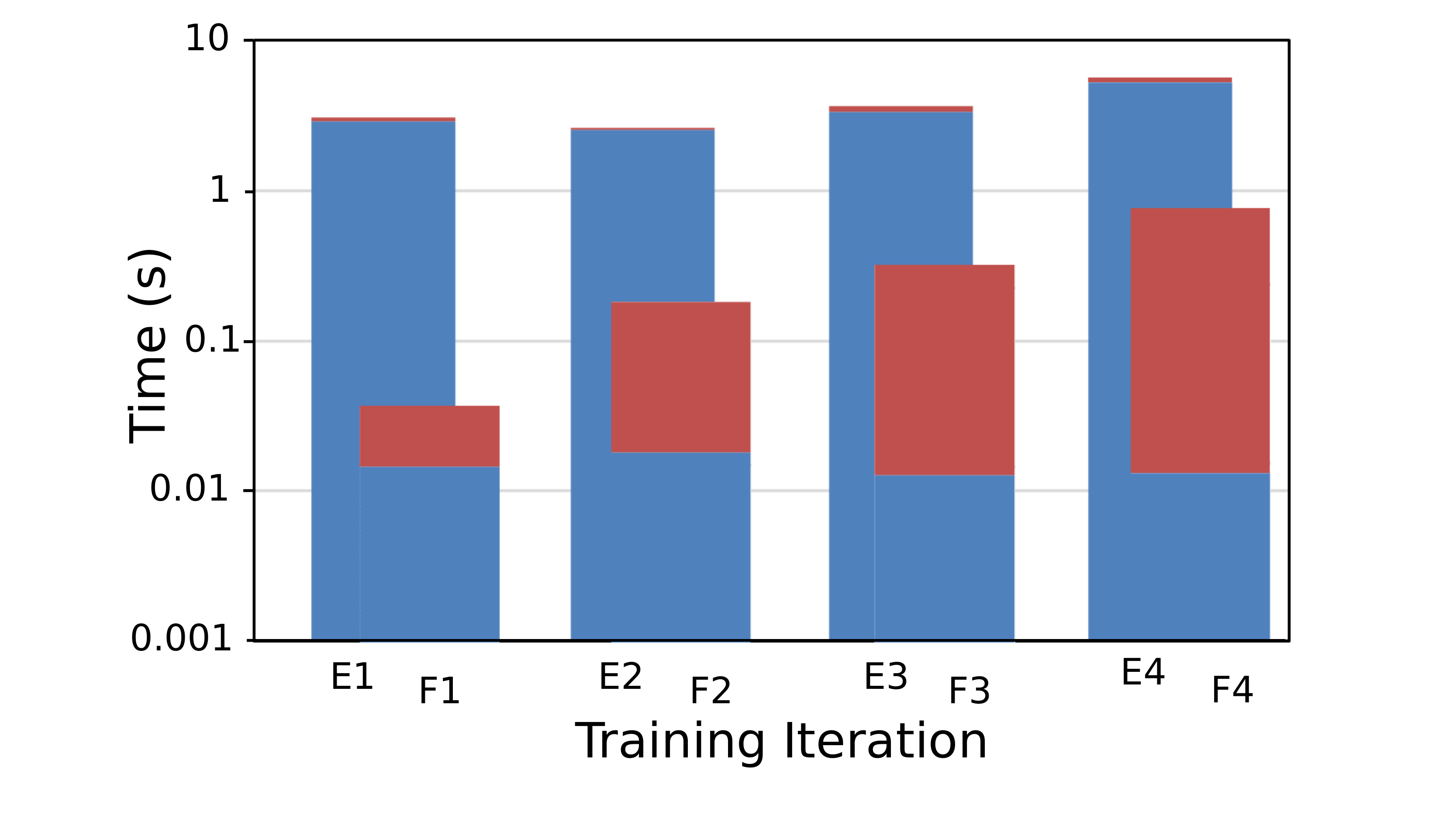}
\caption{Run-times (s) for calculation of Gram matrix (red) and for masking of data (blue) with FLAKE (F) and ESCAPED (E).}
\label{federated}
\end{minipage}
\end{figure}

%\begin{figure}[ht]
%\vskip 0.2in
%\begin{center}
%\centerline{\includegraphics[width=0.6\columnwidth]{pictures/FedLer.pdf}}
%\caption{Comparison of time (s) for calculation of Gram matrix (red) and for masking of data (blue) with FLAKE (F) and ESCAPED (E). In each training iteration the training data set of 1000 data gets updated with additional 1000 data points. Note that a logarithmic scale was used.}
%\label{federated}
%\end{center}
%\vskip -0.2in
%\end{figure}

\subsection{Discussion}

Many privacy-preserving machine learning methods ensure privacy by adding stochasticity, which decreases the result quality (privacy $\sim$ utility trade-off) \citep{chen2005privacy, chen2007towards, lin2015secure, lin2013privacy}. In contrast, the function party in FLAKE obtains an exact Gram matrix (Requirement \emph{Accuracy}), that can be used to compute any desired kernel matrix and later train any kernel-based machine learning algorithm, as if it was centralized data. 
ESCAPED, which provides an accurate solution as well, requires more communication between the parties, which results in longer execution times \citep{unal2021escaped}. As shown in section \ref{sec:experiments}, FLAKE is more efficient due to less communication rounds. Also, FLAKE allows input parties to update the Gram matrix with new samples independently of the previous samples. In ESCAPED, updating the Gram matrix with new samples is not supported. Instead, the Gram matrix must be recomputed using all the samples that the input parties have. After all, FLAKE has various advantages over preceding work using the randomized masking approach.

%With FLAKE, inference and updating is an inexpensive operation, both in terms of memory and execution time (Requirement \emph{Efficiency}). After initial training and optional multiple training iterations, only a small part of the Gram matrix needs to be re-calculated. A large share of these  calculations takes place locally at the input parties. If the input parties do not have the necessary resources to compute the entries, they can also mask the required data and forward them to the function party.

\section{Conclusion}
\label{sec:conclusion}
Federated learning is an essential aspect of distributed machine learning, particularly when data privacy is a primary concern. However, when implementing both Federated Learning and privacy-preserving methods, the quality of model training can suffer as a result. In this work, we have proposed FLAKE, a Federated Learning Approach for KErnel methods, as a solution to that challenge.  Our approach allows for the efficient and private computation of the Gram matrix from data that is distributed on multiple sources, enabling the training of kernel-based machine learning algorithms without any trade-offs in utility. 
Initially, four requirements were formulated, of which we have shown that FLAKE satisfies them: \emph{Privacy}, \emph{Accuracy}, \emph{Updatability} and \emph{Efficiency}. We showed, that FLAKE is both correct and private with regard to the considered threat models. We conducted various experiments on benchmark data sets to show FLAKE meets the accuracy and correctness of centralized models. Besides conducting experiments on well-known data sets, we also replicated the experiments of \citep{unal2019framework} on HIV V3 Loop Sequence data. While other privacy-preserving techniques can be computationally expensive, FLAKE is quite efficient. An analysis of FLAKE and comparable approaches shows, that FLAKE is not as computationally expensive. In order to expand the capabilities of the framework, additional common machine learning operations could be incorporated as future developments. Also, the masking and processing of vertically shared data could be included in FLAKE. 

We believe that FLAKE has the potential to improve healthcare outcomes and reduce costs while addressing the privacy concerns associated with machine learning on clinical data. We also think that it may find many use cases in other application domains that handle sensitive, distributed data. 

%\section*{References}

{
\small

\bibliography{main}

\begin{thebibliography}{}

\bibitem[Adnan et~al., 2022]{adnan2022federated}
Adnan, M., Kalra, S., Cresswell, J.~C., Taylor, G.~W., and Tizhoosh, H.~R.
  (2022).
\newblock Federated learning and differential privacy for medical image
  analysis.
\newblock {\em Scientific reports}, 12(1):1--10.

\bibitem[Aguilera and P{\'e}rez-Aguila, 2004]{aguilera2004general}
Aguilera, A. and P{\'e}rez-Aguila, R. (2004).
\newblock General n-dimensional rotations.

\bibitem[Bonawitz et~al., 2017]{bonawitz2017practical}
Bonawitz, K., Ivanov, V., Kreuter, B., Marcedone, A., McMahan, H.~B., Patel,
  S., Ramage, D., Segal, A., and Seth, K. (2017).
\newblock Practical secure aggregation for privacy-preserving machine learning.
\newblock In {\em proceedings of the 2017 ACM SIGSAC Conference on Computer and
  Communications Security}, pages 1175--1191.

\bibitem[{Center for Machine Learning and Intelligent Systems},
  2023]{UCI-ML-Repository}
{Center for Machine Learning and Intelligent Systems} (2023).
\newblock {UCI Machine Learning Repository}.
\newblock {http://archive.ics.uci.edu}.

\bibitem[Chen and Liu, 2005]{chen2005privacy}
Chen, K. and Liu, L. (2005).
\newblock Privacy preserving data classification with rotation perturbation.
\newblock In {\em Fifth IEEE International Conference on Data Mining
  (ICDM'05)}, pages 4--pp. IEEE.

\bibitem[Chen et~al., 2007]{chen2007towards}
Chen, K., Sun, G., and Liu, L. (2007).
\newblock Towards attack-resilient geometric data perturbation.
\newblock In {\em proceedings of the 2007 SIAM international conference on Data
  mining}, pages 78--89. SIAM.

\bibitem[{European Parliament} and {Council of the European Union}, 2016]{gdpr}
{European Parliament} and {Council of the European Union} (2016).
\newblock Regulation (eu) 2016/679 of the european parliament and of the
  council of 27 april 2016 on the protection of natural persons with regard to
  the processing of personal data and on the free movement of such data, and
  repealing directive 95/46/ec (general data protection regulation).

\bibitem[Evans et~al., 2018]{evans2018pragmatic}
Evans, D., Kolesnikov, V., Rosulek, M., et~al. (2018).
\newblock A pragmatic introduction to secure multi-party computation.
\newblock {\em Foundations and Trends in Privacy and Security}, 2(2-3):70--246.

\bibitem[Haywood et~al., 2021]{haywood2021kernel}
Haywood, A.~L., Redshaw, J., Hanson-Heine, M.~W., Taylor, A., Brown, A., Mason,
  A.~M., Gaertner, T., and Hirst, J.~D. (2021).
\newblock Kernel methods for predicting yields of chemical reactions.
\newblock {\em Journal of Chemical Information and Modeling}.

\bibitem[Lin, 2013]{lin2013privacy}
Lin, K.-P. (2013).
\newblock Privacy-preserving kernel k-means outsourcing with randomized
  kernels.
\newblock In {\em 2013 IEEE 13th International Conference on Data Mining
  Workshops}, pages 860--866. IEEE.

\bibitem[Lin et~al., 2015]{lin2015secure}
Lin, K.-P., Chang, Y.-W., and Chen, M.-S. (2015).
\newblock Secure support vector machines outsourcing with random linear
  transformation.
\newblock {\em Knowledge and Information Systems}, 44:147--176.

\bibitem[Malekzadeh et~al., 2021]{malekzadeh2021dopamine}
Malekzadeh, M., Hasircioglu, B., Mital, N., Katarya, K., Ozfatura, M.~E., and
  G{\"u}nd{\"u}z, D. (2021).
\newblock Dopamine: Differentially private federated learning on medical data.
\newblock {\em arXiv preprint arXiv:2101.11693}.

\bibitem[McMahan et~al., 2017]{mcmahan2017communication}
McMahan, B., Moore, E., Ramage, D., Hampson, S., and y~Arcas, B.~A. (2017).
\newblock Communication-efficient learning of deep networks from decentralized
  data.
\newblock In {\em Artificial intelligence and statistics}, pages 1273--1282.
  PMLR.

\bibitem[Monreale and Wang, 2016]{monreale2016privacy}
Monreale, A. and Wang, W.~H. (2016).
\newblock Privacy-preserving outsourcing of data mining.
\newblock In {\em 2016 IEEE 40th Annual Computer Software and Applications
  Conference (COMPSAC)}, volume~2, pages 583--588. IEEE.

\bibitem[Morota and Gianola, 2014]{morota2014kernel}
Morota, G. and Gianola, D. (2014).
\newblock Kernel-based whole-genome prediction of complex traits: a review.
\newblock {\em Frontiers in genetics}, 5:363.

\bibitem[Mugunthan et~al., 2019]{mugunthan2019smpai}
Mugunthan, V., Polychroniadou, A., Byrd, D., and Balch, T.~H. (2019).
\newblock Smpai: Secure multi-party computation for federated learning.
\newblock In {\em Proceedings of the NeurIPS 2019 Workshop on Robust AI in
  Financial Services}.

\bibitem[Ookla, 2022]{Ookla2022}
Ookla (2022).
\newblock Speedtest global index.
\newblock \url{https://www.speedtest.net/global-index}.
\newblock Accessed: 01 20, 2022.

\bibitem[Pfitzner et~al., 2021]{pfitzner2021federated}
Pfitzner, B., Steckhan, N., and Arnrich, B. (2021).
\newblock Federated learning in a medical context: a systematic literature
  review.
\newblock {\em ACM Transactions on Internet Technology (TOIT)}, 21(2):1--31.

\bibitem[Qiu et~al., 2008]{qiu2008protecting}
Qiu, L., Li, Y., and Wu, X. (2008).
\newblock Protecting business intelligence and customer privacy while
  outsourcing data mining tasks.
\newblock {\em Knowledge and Information Systems}, 17(1):99--120.

\bibitem[Stripelis et~al., 2021]{stripelis2021secure}
Stripelis, D., Saleem, H., Ghai, T., Dhinagar, N., Gupta, U., Anastasiou, C.,
  Ver~Steeg, G., Ravi, S., Naveed, M., Thompson, P.~M., et~al. (2021).
\newblock Secure neuroimaging analysis using federated learning with
  homomorphic encryption.
\newblock In {\em 17th International Symposium on Medical Information
  Processing and Analysis}, volume 12088, pages 351--359. SPIE.

\bibitem[{\"U}nal et~al., 2019]{unal2019framework}
{\"U}nal, A.~B., Akg{\"u}n, M., and Pfeifer, N. (2019).
\newblock A framework with randomized encoding for a fast privacy preserving
  calculation of non-linear kernels for machine learning applications in
  precision medicine.
\newblock In {\em International Conference on Cryptology and Network Security},
  pages 493--511. Springer.

\bibitem[{\"U}nal et~al., 2021]{unal2021escaped}
{\"U}nal, A.~B., Akg{\"u}n, M., and Pfeifer, N. (2021).
\newblock Escaped: Efficient secure and private dot product framework for
  kernel-based machine learning algorithms with applications in healthcare.
\newblock In {\em Proceedings of the AAAI Conference on Artificial
  Intelligence}, volume~35, pages 9988--9996.

\bibitem[Wibawa et~al., 2022]{wibawa2022homomorphic}
Wibawa, F., Catak, F.~O., Kuzlu, M., Sarp, S., and Cali, U. (2022).
\newblock Homomorphic encryption and federated learning based
  privacy-preserving cnn training: Covid-19 detection use-case.
\newblock In {\em Proceedings of the 2022 European Interdisciplinary
  Cybersecurity Conference}, pages 85--90.

\bibitem[Wolberg et~al., 1992]{wolberg1992breast}
Wolberg, W.~H., Street, W.~N., and Mangasarian, O.~L. (1992).
\newblock Breast cancer wisconsin (diagnostic) data set.
\newblock {\em UCI Machine Learning Repository [http://archive. ics. uci.
  edu/ml/]}.

\bibitem[Yu et~al., 2006]{yu2006privacy}
Yu, H., Jiang, X., and Vaidya, J. (2006).
\newblock Privacy-preserving svm using nonlinear kernels on horizontally
  partitioned data.
\newblock In {\em Proceedings of the 2006 ACM symposium on Applied computing},
  pages 603--610.

\bibitem[Zeng et~al., 2008]{zeng2008fast}
Zeng, Z.-Q., Yu, H.-B., Xu, H.-R., Xie, Y.-Q., and Gao, J. (2008).
\newblock Fast training support vector machines using parallel sequential
  minimal optimization.
\newblock In {\em 2008 3rd international conference on intelligent system and
  knowledge engineering}, volume~1, pages 997--1001. IEEE.

\bibitem[Zhang et~al., 2022]{zhang2022augmented}
Zhang, C., Ekanut, S., Zhen, L., and Li, Z. (2022).
\newblock Augmented multi-party computation against gradient leakage in
  federated learning.
\newblock {\em IEEE Transactions on Big Data}.

\bibitem[Zhang et~al., 2020]{zhang2020batchcrypt}
Zhang, C., Li, S., Xia, J., Wang, W., Yan, F., and Liu, Y. (2020).
\newblock Batchcrypt: Efficient homomorphic encryption for cross-silo federated
  learning.
\newblock In {\em Proceedings of the 2020 USENIX Annual Technical Conference
  (USENIX ATC 2020)}.

\end{thebibliography}
\bibliographystyle{apalike}

\medskip
}

\section{Supplementary Material}

\subsection{Privacy Proof}
The following proof is based on a proof by \cite{unal2019framework}.

\begin{theorem} \label{theorem4}
FLAKE provides security against a malicious function party A, assuming A is either semi-honest or malicious and does not collude with any input parties. In this scenario, A is unable to deduce the data of the input parties from the Gram matrix $G$ that is generated as a result.
\end{theorem}
\begin{proof}

Although the number of features are hidden by FLAKE, we assume now the full Gram matrix $G=DD^T$ with the data of the input parties $D=[A, B, C]$ and the number of features are known to the function party. We show, that an attacker could not obtain any data since it there are multiple matrices that result in the Gram matrix. 

Assume that there is a rotation matrix $R \in \mathbb{R}^{N \times N}$ where $N=2(n_a + n_b + n_c)$ with $n_x$ is the number of samples in the corresponding party. Then, there is a matrix $E$ which can be computed by $E = R^{-1}D$. From that, we can say that $D=RE$. Then, due to the rotation property of $R^{-1}=R^T$, the the following holds:

\begin{equation*}
    \begin{split}
        K & = D^T D \\
        & = (RE)^T(RE)\\
        &=E^TR^TRE \\
        &= E^TR^{-1}RE\\
        &= E^TE\\
    \end{split}
\end{equation*}
\end{proof}

Since \cite{aguilera2004general} showed, that countless rotation matrices can be generated, we cannot obtain a unique matrix resulting in Gram matrix $G$: For every new rotation matrix $\theta \in \mathbb{R}^{N \times N}$, there exists a new matrix $\beta = \theta^{-1}D$ satisfying $G = \beta^T\beta$. Thus, A is unable to deduce the input parties' data $D=(A, B, C)$ from $G=D^DT$.

\newpage

\subsection{Supplementary Experiments}

All methods employed a polynomial kernel and identical hyperparameter settings. For this implementation, Sequential Minimal Optimization (libsvm) provided by scikit-learn was used \cite{zeng2008fast}. Since the Pima Indian diabetes data set, HIV and Breast Cancer data set have an unbalanced distribution of classes, we have applied Macro Averaging. Correspondingly, for the balanced synthetic data set, Micro Averaging.

\begin{table}[h]
\caption{statistics of data sets used in the experiment section}
\label{scores}
%\vskip 0.15in
\begin{center}
\begin{scriptsize}
\begin{sc}
\begin{tabular}{lccccr}
\toprule
& \multicolumn{2}{c}{Naive} & \multicolumn{2}{c}{FLAKE} \\
\midrule
Data set & Number of data points & Number of Features & binary/multi - label & distribution  \\
\midrule
Diabetes    & 768 & 8 & binary & in-balanced \\
Cancer  & 569 & 10 & binary & in-balanced \\
HIV     & 766 & 924 &  binary & in-balanced \\
Synthetic    & 500-8000 & 20 & multi class & balanced       \\
\bottomrule
\end{tabular}
\end{sc}
\end{scriptsize}
\end{center}
%\vskip -0.1in
\end{table}

\begin{table}[h]
\caption{ROC AUC with standard deviation for Naive SVM, FLAKE, ESCAPED, RSVM, PPSVM on various data sets.}
\label{scores}
%\vskip 0.15in
\begin{center}
\begin{scriptsize}
\begin{sc}
\begin{tabular}{lccccc}
\toprule
 & Naive & FLAKE & ESCAPED & RSVM & PPSVM \\
\midrule
Diabetes        & 0.97$\pm$ 0.04 & 0.97$\pm$ 0.04 & 0.97$\pm$ 0.04 & 0.95$\pm$ 0.02 & 0.94$\pm$ 0.04\\
Cancer  & 0.97 $\pm$ 0.03 & 0.97 $\pm$ 0.03 & 0.97 $\pm$ 0.03 & 0.96 $\pm$ 0.02 & 0.97 $\pm$ 0.04 \\
HIV       & 0.78$\pm$ 0.03 & 0.78$\pm$ 0.03 & 0.78$\pm$ 0.03 & 0.65$\pm$ 0.17 & 0.64$\pm$ 0.10\\
Synthetic  & 0.97$\pm$ 0.01 & 0.97$\pm$ 0.01 & 0.97$\pm$ 0.01 & 0.83$\pm$ 0.04 & 0.95$\pm$ 0.01\\
\bottomrule
\end{tabular}
\end{sc}
\end{scriptsize}
\end{center}
%\vskip -0.1in
\end{table}

\begin{table}[h]
\caption{Overhead (Masking time + Gram time) for FLAKE, ESCAPED, RSVM, PPSVM for three input parties with 1000 dp each.}
\label{scores}
%\vskip 0.15in
\begin{center}
\begin{scriptsize}
\begin{sc}
\begin{tabular}{lccccc}
\toprule
& FLAKE & ESCAPED & RSVM & PPSVM \\
\midrule
Masking time for one IP      & 0.00146 &  1.23610 & 0.00201 &  0.02257\\
Time to compute Gram   & 0.02071 & 0.03156 & 0.00530 & 0.01121 \\
\bottomrule
\end{tabular}
\end{sc}
\end{scriptsize}
\end{center}
%\vskip -0.1in
\end{table}

%%%%%%%%%%%%%%%%%%%%%%%%%%%%%%%%%%%%%%%%%%%%%%%%%%%%%%%%%%%%

\end{document}